\documentclass[letterpaper, 10 pt, conference]{style/ieeeconf} 
\IEEEoverridecommandlockouts 
\usepackage{latexsym}
\usepackage{amsmath}                
\usepackage{amssymb}                
\usepackage{amsfonts}               
\usepackage{algorithm}              
\usepackage{algorithmicx}           
\usepackage[noend]{algpseudocode}   
\usepackage{graphicx}               
\usepackage{float}                  
\usepackage{epstopdf}               
\usepackage[abs]{overpic}           
\usepackage{pict2e}                 

\usepackage{color}                  
\usepackage{array}                  
\usepackage{multirow}               
\usepackage{tikz}                   
\usetikzlibrary{automata,bayesnet}  

\usepackage{pgfplots}               
\pgfplotsset{compat = newest}

\newtheorem{theorem}{Theorem}

\newtheorem{proposition}{Proposition}

\DeclareMathOperator*{\argmax}{argmax}

\newcommand{\verticaledge}[3]{\draw[->] (#1) to node[right]{#2} (#3)}
\newcommand{\horizontaledge}[3]{\draw[->] (#1) to node[above]{#2} (#3)}

\newcommand{\forwardslantedgebelow}[3]{\draw[->] (#1) to node[below right]{#2} (#3)}
\newcommand{\backwardslantedgebelow}[3]{\draw[->] (#1) to node[below left]{#2} (#3)}

\algdef{SE}[DOWHILE]{Do}{doWhile}{\algorithmicdo}[1]{\algorithmicwhile\ #1}%

\renewcommand{\vec}{\mathbf}

\newcommand{\maximize}[1]{\underset{#1}{\text{maximize}}}

\newcommand{\subjectto}{\text{subject to}}

\algrenewcommand{\algorithmiccomment}[1]{\hfill\textit{// #1}}

\title{\LARGE \bf%
Multi-Objective Policy Gradients with Topological Constraints
}
\author{Kyle Hollins Wray,$^{*1}$ Stas Tiomkin,$^{*2}$ Mykel J. Kochenderfer,$^1$ and Pieter Abbeel$^3$
    \thanks{$^*$ Equal contribution.}%
    \thanks{$^1$ Stanford University, Stanford, CA, USA.}
    \thanks{$^2$ San Jose State University, CA, USA.}
    \thanks{$^3$ University of California, Berkeley, CA, USA.}
    \thanks{{\tt\small kylewray@stanford.edu}, %
            {\tt\small stas.tiomkin@sjsu.edu}, %
            {\tt\small mykel@stanford.edu}, and %
            {\tt\small pabbeel@cs.berkeley.edu}}
}

\begin{document}

\bstctlcite{IEEEdisabledash:BSTcontrol}

\maketitle
\thispagestyle{empty}
\pagestyle{empty}


\begin{abstract}
    Multi-objective optimization models that encode ordered sequential constraints provide a solution to model various challenging problems including encoding preferences, modeling a curriculum, and enforcing measures of safety.
    A recently developed theory of topological Markov decision processes (TMDPs) captures this range of problems for the case of discrete states and actions.
    In this work, we extend TMDPs towards continuous spaces and unknown transition dynamics by formulating, proving, and implementing the policy gradient theorem for TMDPs.
    This theoretical result enables the creation of TMDP learning algorithms that use function approximators, and can generalize existing deep reinforcement learning (DRL) approaches.
    Specifically, we present a new algorithm for a policy gradient in TMDPs by a simple extension of the proximal policy optimization (PPO) algorithm.
    We demonstrate this on a real-world multiple-objective navigation problem with an arbitrary ordering of objectives both in simulation and on a real robot.
\end{abstract}

\section{Introduction}

In recent years, the theoretical foundations of Markov decision processes (MDPs)~\cite{Bbook57} and reinforcement learning (RL)~\cite{SBbook18} algorithms have grown to practical robotic applications in domains ranging from autonomous helicopter flight~\cite{ANicml04} to autonomous vehicles~\cite{WWZijcai17}.
However, larger real-world domains often must consider multiple objectives such as in energy, comfort, and noise management in buildings~\cite{KKKJBVTaic12} and hybrid electric vehicles~\cite{WLPiv21}.
In reinforcement learning, approaches like curriculum learning~\cite{FHWZAcorl17} and incorporating safety considerations~\cite{WSicml20} often sequentially learn differing objectives to incrementally develop skills.
The topological MDP (TMDP)~\cite{Wthesis19,WZMaaai15} is a general model that captures this space of problems.
There are various methods for multi-objective learning strategies such as constrained policy optimization~\cite{AHTAicml17} for non-incremental constrained MDPs (CMDPs) for related problems like robot walking~\cite{HHZTTLarxiv18}.
However, there exists a gap in the theoretical foundation for solving these general multi-objective models with policy gradient-based algorithms, consequently limiting the principled development of deep RL solutions.

In a TMDP, multiple objectives are ordered in a directed acyclic graph (DAG)~\cite{Wthesis19}.
The ordering can capture preference, constraint, curriculum, or safety.
A slack term, or equivalently budget, is assigned to each edge in the DAG.
It represents the allowable deviation in value of an optimal policy in a parent objective to improve the value of a child objective.
This flexible structure generalizes the constrained MDP (CMDP)~\cite{Abook99}, which is a fan-structured DAG, and a lexicographic MDP (LMDP)~\cite{WZMaaai15}, which is a chain-structured DAG (Figure~\ref{fig:tmdp-examples}).
The agent can learn all the objectives simultaneously~\cite{Abook99}, or it can incrementally step down the DAG in the order implied by its edges and learn each parent objective before moving on to its children~\cite{WZMaaai15}.


Consider a robot navigation domain where the robot can be tasked with navigating to a goal location, monitoring a particular room, and avoiding another room, all based on a customer's preferences.
Each customer may have a preference, such as for example prioritizing monitoring over navigation or navigation over avoiding a region.
In each case, they may also have a tolerance allowing one objective to be reduced in favor of improving another.
These preference structures and tolerances can be captured within a TMDP and provided to the agent as it learns to complete its objectives.

In practice, solving these kinds of multi-objective problems optimally is not feasible because all of the objectives depend on the current policy~\cite{GKSicml98,WZMaaai15,AHTAicml17}.
Chain-structured approaches that modify the value have been considered in the tabular case~\cite{GKSicml98}.

Local action restriction (LAR)~\cite{WZMaaai15,WZijcai15} is a scalable approximate method to solve TMDPs.
It moves the global constraint into a set of local constraints over the states.
This reformulation removes the dependence of the policy on the constraints' values.
Fan-structured approaches (Figure~\ref{fig:tmdp-examples} (b)) have considered policy gradients, such as in CPO~\cite{AHTAicml17}.
However, the method only works for CMDPs and cannot model a DAG's topologically ordered constraints as in TMDPs.
Also, it requires computing expensive second-order terms as a Hessian to approximate constraint satisfaction (see Section 6.1~\cite{AHTAicml17}).
In general, these approaches are computationally expensive and/or do not admit a direct policy gradient.


\emph{The goal of this work is to provide a multiple objective generalization of the policy gradient theorem for any given DAG of constraints and slack values.}
Our theorem derives a policy gradient that performs an accurate optimization of the objectives, each being subject to the constraints of their ancestors. 
We apply this general result in the specific case of extending proximal policy optimization (PPO)~\cite{SWDRKarxiv17} for use in TMDPs.
The resulting new algorithm is called topological policy optimization (TPO), integrating the new policy gradient into the PPO objective.

Our main contributions include:
(1) a formal definition and derivation a multi-objective policy gradient within a TMDP (Section~\ref{sec:policy-gradient-theorem-for-tmdp});
(2) the deep reinforcement learning algorithm called TPO using this new policy gradient (Section~\ref{sec:topological-policy-optimization});
(3) experiments demonstrating the success of TPO in both simulation and on a real robot (Section~\ref{sec:experiments}).

\newpage

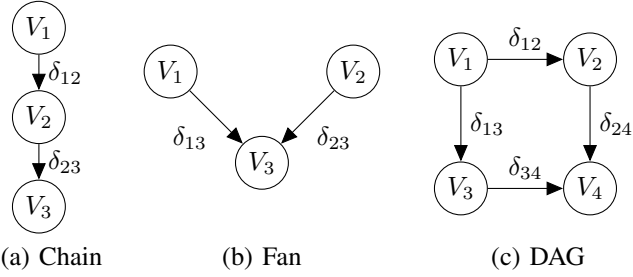
\begin{figure}[ht]
    \newcolumntype{S}{>{\centering\arraybackslash} m{1.5cm}}
    \newcolumntype{L}{>{\centering\arraybackslash} m{3.25cm}}
    \begin{tabular}{SLL}
        \begin{tikzpicture}
            \node[latent] (1) {$V_1$} ; %
            \node[latent, below=of 1, yshift=1.5em] (2) {$V_2$} ; %
            \node[latent, below=of 2, yshift=1.5em] (3) {$V_3$} ; %
            \verticaledge{1}{$\delta_{12}$}{2} ; %
            \verticaledge{2}{$\delta_{23}$}{3} ; %
        \end{tikzpicture} &
        \begin{tikzpicture}
            \node[latent] (3) {$V_3$} ; %
            \node[latent, above left=of 3] (1) {$V_1$} ; %
            \node[latent, above right=of 3] (2) {$V_2$} ; %
            \backwardslantedgebelow{1}{$\delta_{13}$}{3} ; %
            \forwardslantedgebelow{2}{$\delta_{23}$}{3} ; %
        \end{tikzpicture} &
        \begin{tikzpicture}
            \node[latent] (1) {$V_1$} ; %
            \node[latent, right=of 1] (2) {$V_2$} ; %
            \node[latent, below=of 1] (3) {$V_3$} ; %
            \node[latent, right=of 3] (4) {$V_4$} ; %
            \verticaledge{1}{$\delta_{13}$}{3} ; %
            \horizontaledge{3}{$\delta_{34}$}{4} ; %
            \verticaledge{2}{$\delta_{24}$}{4} ; %
            \horizontaledge{1}{$\delta_{12}$}{2} ; %
        \end{tikzpicture} \\
        (a) Chain & (b) Fan & (c) DAG
    \end{tabular}
    \caption{Example TMDP objective $V_i$ nodes and edges with slacks $\delta_{ij}$.}
    \label{fig:tmdp-examples}
    \vspace{-1em}
\end{figure}

\section{Background}
\label{sec:model}

\subsection{Markov Decision Process (MDP)}

A \textbf{Markov decision process (MDP)} is defined by the tuple $\langle S, A, T, R \rangle$.
$S$ is a set of states.
$A$ is a set of actions.
$T(s' \mid s, a)$ is a Markovian state transition that captures the probability of transitioning to a successor state $s'$ given that action $a$ was taken in state $s$. 
$R(s, a)$ is a reward function that describes the effect of performing an action $a$ in a state $s$. 
For discounted MDPs, $\gamma \in [0, 1)$ is a discount factor on the reward over time.
It is also common to have an initial $s^0$.

Reinforcement learning can be described as the collection of algorithms that do not assume $T$ is provided~\cite{SBbook18}. Moreover, many of these algorithms do not assume $S$, $A$, or $R$ are fully provided a priori either. Instead, they are only observed upon visiting a state and performing an action.

A \textbf{policy} $\pi(a \mid s)$ maps each state $s$ to a probability of performing each action $a$.
Let $\tau = \langle s^0, a^0, r^0, s^1, a^1, r^1, \ldots \rangle$ denote a \textbf{trajectory} with each state $a^t \sim \pi(\cdot \mid s^t)$, $r^t = R(s^t, a^t)$, and $s^{t+1} \sim T(\cdot \mid s^t, a^t)$.
The goal in reinforcement learning is to explore and use experienced trajectories $\tau$ to find a policy $\pi^*$ that maximizes expected reward.

Thus, the agent seeks an optimal policy $\pi^*$ such that:
\begin{equation*}
    \pi^* = \underset{\pi}{\argmax} ~ \mathbb{E} \Big[ \sum_{t=0}^\infty \gamma^t R(s^t, a^t) \mid \pi, s^0 \Big].
\end{equation*}
Let the \textbf{value} $V^\pi : S \rightarrow \mathbb{R}$ of a policy $\pi$ be its expected reward.
Given the model, we can solve the MDP by iteratively applying the Bellman optimality equation at states $s$:
\begin{equation*}
    V^\pi(s) = \max_a Q^\pi(s, a)
\end{equation*}
\begin{equation*}
    Q^\pi(s, a) = R(s, a) + \gamma \sum_{s'} T(s' \mid s, a) V^\pi(s')
\end{equation*}
with $V^*$ and $Q^*$ denoting optimal values.
The \textbf{advantage} is defined as $A^\pi(s, a) = Q^\pi(s, a) - V^\pi(s)$~\cite{SWDRKarxiv17} and is useful in describing the advantage in value of one action over another.

\subsection{Policy Gradient}

Policy gradient methods improve the policy along a gradient and forms the foundation for function approximation~\cite{SMSMnips00}.
It assumes the policy $\pi$ is parameterized by parameters $\theta$.
Let $\rho(\pi) = V^\pi(s^0)$ be the value of the initial state $s^0$ following a policy $\pi$.
The policy gradient is:
\begin{equation*}
    \frac{\partial \rho^\pi}{\partial \theta} = \sum_s d^\pi(s) \sum_a \frac{\partial \pi(a \mid s)}{\partial \theta} Q^\pi(s, a)
\end{equation*}
with $d^\pi$ denoting the stationary state distribution of $\pi$.

\subsection{Topological Markov Decision Process}

A \textbf{topological Markov decision process (TMDP)}~\cite{Wthesis19} generalizes the MDP to multiple objectives and is defined by $\langle S, A, T, {\vec R}, E, {\vec \delta} \rangle$.
${\vec R} : S \times A \rightarrow \mathbb{R}^k$ emits $k$ rewards; each $i \in K = \{1, \ldots, k\}$ can be written as $R_i(s, a)$.
$E \subseteq K \times K$ forms a directed acyclic graph (DAG) over the rewards, with one leaf node, assumed to be $k$ without loss of generality.
${\vec \delta} = \{\delta_{wv} | (w, v) \in E\}$ denotes the set of slack variables---allowable deviation from optimal value---for each parent-child objective pair.
The objective in a TMDP is: for each objective $i \in K$, following the order of the DAG $E$, we solve:
\begin{align}
    \maximize{\pi}  & \quad V_i^\pi(s^0) \label{eq:top:lagrangian}\\
    \subjectto      & \quad V_w^*(s^0) - V_w^\pi(s^0) \leq \delta_{wv}, ~~ \forall v \in \mathcal{A}_i \cup \{i\}, \forall w \in \mathcal{P}_v\nonumber
\end{align}
with $\mathcal{P}_v, \mathcal{A}_i \subset K$ denoting the parents and ancestors of $v$ and $i$ in $E$ respectively; and $V_j^*(s^0)$ denoting the optimal value of $j$ following this same constrained objective.
In other words, the constraints state that every vertex $v$ must satisfy the slacks $\eta_{wv}$ from its parents $w$ following an edge $e=(w, v)$.
Also, let $E_i = \{e = \langle w, v \rangle \in E \mid v \in \mathcal{A}_i \cup \{i\} \text{ and } w \in \mathcal{P}_v\}$ denote all $i$'s ancestral edges in $E$.
An optimal policy $\pi^*$ is the policy $\pi^*_k$ computed by the leaf node $k$ in $E$.
The leaf optimizes $V_k^*(s^0)$ ensuring that the union of all ancestral constraints are satisfied.

\textbf{Local action restriction (LAR)}~\cite{WZMaaai15,Wthesis19} is an approximation that restriction actions locally by a local slack ${\vec \eta}$, rather than a global slack ${\vec \delta}$.
It has been shown that if $\eta_i = (1 - \gamma) \delta_i$, then global slack can be preserved, if desired.
LAR enables scalable algorithms to be devised. The resulting TMDP objective for each objective $i$ becomes:
\begin{align}
    \maximize{\pi}  & \quad V_i^\pi(s^0) \label{eq:tmdp:objective-lar} \\
    \subjectto      & \quad V_w^*(s) - Q_w^*(s, \pi(s)) \leq \eta_{wv}, \nonumber \\
                    & \quad \quad \quad \quad \quad \forall v \in \mathcal{A}_i \cup \{i\}, \forall w \in \mathcal{P}_v, s \in S \nonumber
\end{align}
TMDPs generalize both LMDPs and feasible CMDPs.
See Figure~\ref{fig:tmdp-examples} for examples. 

\section{Policy Gradient Theorem for TMDP}
\label{sec:policy-gradient-theorem-for-tmdp}


Solving the LAR optimization in Equation~\ref{eq:tmdp:objective-lar} can be done using a modified Bellman optimality equation for tabular policy~\cite{WZMaaai15,WZijcai15,Wthesis19}.
This formulation does not admit function approximation, or specifically, the use of a policy gradient or deep reinforcement learning methods.
Our main result in this section is a formulation and proof of a policy gradient theorem to solve Equation~\ref{eq:tmdp:objective-lar}.

The goal is to prove a TMDP policy gradient theorem.
To accomplish this, we need to formulate a constrained Bellman optimality equation corresponding to Equation~\ref{eq:tmdp:objective-lar}.
LAR allows us to move the constraints into the Bellman optimality equation~\cite{Wthesis19}.
However, this becomes a constrained optimization problem.
We seek to reduce this to an unconstrained optimization problem in order to leverage the abundant prior work on (approximate) policy gradients~\cite{SLMJAicml15,SWDRKarxiv17}.
We convert the constrained Bellman equation into an unconstrained Bellman equation using Lagrange multipliers.
However, the standard approach results in the reward itself being penalized by an arbitrary amount based on the ancestors' advantages.
Consequently, this would change the optimization's value function. 

We provide a solution that will not affect the reward in this manner. This requires both: (1) the Lagrange multiplier to be assigned properly via a bound, and (2) the constraint term to be transformed such that it is zero when the constraint is not violated.
We derive a bound on the Lagrange multiplier to ensure it is large enough such that any action that violates the constraint will not be chosen.
We prove that the transformed constraint term preserves the original constraint satisfaction.
When these facts are combined, the result is that the original value is preserved and the constraints are satisfied.

With this new Bellman optimality equation, Lagrange multiplier bounds, and transformed constraint, we prove the constrained policy gradient theorem for the TMDP.

\subsection{Lagrangian Bellman Optimality Equation}


We need to derive the corresponding Bellman optimality equation.
However directly writing a Bellman equation for the unconstrained optimization formulation in Equation~\ref{eq:top:lagrangian} can be onerous due to the extra terms.
Instead, as shown in prior work~\cite{WZMaaai15,Wthesis19}, we can leverage LAR to move the constraints in Equation~\ref{eq:tmdp:objective-lar} into the Bellman optimization problem over actions.
Ensuring the original constraints are satisfied at each state, implies they are satisfied for the original optimization.
This Bellman optimality equation with LAR at a state $s$ is the constrained optimization~\cite{Wthesis19}:
\begin{align}
    \maximize{a}    &\quad Q_i^\pi(s, a) = R_i(s, a) + \gamma \sum_{s'} T(s' \mid s, a) V_i^\pi(s') \nonumber \\
    \subjectto      &\quad {-}A_w^*(s, a) \leq \eta_{wv}, \forall w,v \label{eq:tmdp:bellman-lar}
\end{align}
with $V_i^\pi(s) = Q_i^\pi(s, a^*)$ for constraint-optimal action $a^*$.

The naive use of Lagrange multipliers would result in the undesirable Equation~\ref{eq:tmdp:bad-naive-lagrangian-equation} below that arbitrarily modifies the rewards, affecting the values as a consequence:
\begin{align}
    \textbf{V}_i^\pi(s) &= \max_{a} \Big( R_i(s, a) + \gamma \sum_{s'} T(s' \mid s, a) \textbf{V}_i^\pi(s') \nonumber \\
                        &\quad\quad\quad - \sum_{v,w} \beta_{wvs} ({-}A_w^*(s, a) - \eta_{wv}) \Big) \label{eq:tmdp:bad-naive-lagrangian-equation}
\end{align}
with Lagrange multipliers $\beta_{wvs} \geq 0$ and Lagrangian value function $\textbf{V}_i^\pi$.
To illustrate the issue, consider two cases.
If objective $i$ chooses an action beyond the budgeted slack $\eta_{wv}$, then there should be a penalty.
Otherwise, there should not be an increase in the reward for an action less than the budgeted slack.
Instead, the agent should simply maximize its original reward.
Thus, crucially, the extra constraint terms should only be a penalty if the constraint is violated.
Otherwise, the reward should remain unaffected.
To accomplish this, we need to transform the constraint as follows.

\begin{proposition}
    \label{prop:lagrangian-bellman-solves-tmdp-lar}
    For an objective $i$ and state $s$, the optimization in Equation~\ref{eq:tmdp:bellman-lar} is equivalently solved by:
    \begin{align}
        \maximize{a}    &\quad Q_i^\pi(s, a) = R_i(s, a) + \gamma \sum_{s'} T(s' \mid s, a) V_i^\pi(s') \nonumber \\
        \subjectto      &\quad C_{wv}(s, a) \leq 0, \forall w,v \label{eq:tmdp:bellman-lar-transformed}
    \end{align}
    with $C_{wv}(s, a) = \max\{0, {-}A_w^*(s, a) - \eta_{wv}\}$.
\end{proposition}

\begin{proof}
The TMDP Bellman optimality equation with LAR at a state $s$ follows Equation~\ref{eq:tmdp:bellman-lar}.
By definition, for the optimal policy's values, ${-}A_w^*(s, a) = V_w^*(s) - Q_w^*(s, a) \geq 0$ and there always exists an action $a^*$ such that ${-}A_w^*(s, a) = 0$.
Thus this optimization always feasible and we have:
\begin{align*}
    {-}A_w^*(s, a) \leq \eta_{wv} &\Rightarrow {-}A_w^*(s, a) - \eta_{wv} \leq 0 \\
                                  &\Rightarrow \max\{0, {-}A_w^*(s, a) - \eta_{wv}\} \leq 0 \\
                                  &\Rightarrow C_{wv}(s, a) \leq 0
\end{align*}
This lets us rewrite the optimization problem as in Equation~\ref{eq:tmdp:bellman-lar-transformed}, yielding our result.
\end{proof}

Now we can compute the Lagrangian of the constrained optimization problem in Equation~\ref{eq:tmdp:bellman-lar-transformed}.
The \textbf{Lagrangian Bellman optimality equation} is:
\begin{align}
    \textbf{V}_i^\pi(s) &= \max_{a} \Big( R_i(s, a) + \gamma \sum_{s'} T(s' \mid s, a) \textbf{V}_i^\pi(s') \nonumber \\
                        &\quad\quad\quad - \sum_{v,w} \beta_{wvs} C_{wv}(s, a) \Big) \label{eq:tpo:lagrangian-bellman}
\end{align}
For notational convenience, for any objective $j \in K$, the \textbf{Lagrangian Q-value} is $\mathbf{Q}_j(s, a)$ and the \textbf{Lagrangian advantage} is: $\mathbf{A}_j(s, \pi(s)) = \mathbf{Q}_j(s, \pi(s)) - \mathbf{V}_j(s)$.
In practice, $\mathbf{A}_j$ can be computed using a new form of generalized advantage, discussed in the next section.

Equation~\ref{eq:tpo:lagrangian-bellman} has two main properties.
First, it converted the original constrained optimization problem into an unconstrained Bellman optimality equation.
Second, if its constraints are satisfied, then it preserves the original reward; otherwise it penalizes the reward.

The scale of the constraint terms $C_{wv}(s, a)$ can be arbitrary based on the scale of the values.
If we do not select a sufficiently large multiplier, then it is possible that the penalty will not be enough to ensure constraint-violating actions are never chosen.

Thus, we derive a lower bound on $\beta_{wvs}$ such that the application of the Lagrangian Bellman optimality equation (Equation~\ref{eq:tpo:lagrangian-bellman}) will never select constraint-violating actions.
This is derived in Proposition~\ref{prop:lagrange-multiplier-bound}.

Finally, we must also prove that it preserves optimality of the original problem.
In summary, the use of Equation~\ref{eq:tpo:lagrangian-bellman} with $\beta_{wvs}$ as described below will converge to the correct values and simultaneously enforce the LAR constraints.
This is proven in Theorem~\ref{prop:lagrangian-bellman-solves-tmdp-lar-optimization}.

\begin{proposition}
    \label{prop:lagrange-multiplier-bound}
    Given objective $i$, state $s$, ancestor edges $\langle w, v \rangle \in E_i$, values $Q_i$, at least one constraint-violating action (infeasibility) exists, and optimal constraint-satisfying action $\hat{a}^*$,
    if all $\beta_{wvs}$ satisfy:
    \begin{equation}
        \label{eq:tmdp:lower-bound}
        \beta_{wvs} \geq \max_a \left\{ \begin{array}{l l}
                            \frac{Q_i(s, a) - Q_i(s, \hat{a}^*)}{\sum_{v,w} C_{wv}(s, a)} & \quad \text{if } \sum_{v,w} C_{wv}(s, a) > 0 \\
                            0 & \quad \text{otherwise}
                        \end{array} \right.
    \end{equation}
    then no constraint-violating actions $\hat{a}$ will be chosen, i.e., $\max_a \mathbf{Q}_i^\pi(s, a) = \mathbf{Q}_i^\pi(s, \hat{a}^*)$.
\end{proposition}

\begin{proof}
    At any iteration of the Bellman equation, for any states $s'$, assume no other states have chosen suboptimal action.
    By Proposition~\ref{prop:lagrangian-bellman-solves-tmdp-lar} and Equation~\ref{eq:tpo:lagrangian-bellman}, the Lagrangian values are identical to the constrained optimization problems' values $\mathbf{V}_i^\pi(s') = V_i^\pi(s')$.

    Below, let $Q_i(s, a)$ be the result of applying the standard Bellman equation.
    Let $\hat{a}^*$ be the optimal constraint-satisfying action solving Equation~\ref{eq:tmdp:bellman-lar}.
    Note that both $Q_i(s, a)$ and $\hat{a}^*$ can be computed separately without depending on the equation being considered (Equation~\ref{eq:tpo:lagrangian-bellman}).

    Consider the application of the Lagrangian Bellman optimality equation on a state $s$ (Equation~\ref{eq:tpo:lagrangian-bellman}).
    Assume the existence of a constraint violating action $\hat{a}$. (By construction in Proposition~\ref{prop:lagrangian-bellman-solves-tmdp-lar}, any constraint-satisfying action will result in the penalty constraint terms being $0$.)
    For this maximization, it is sufficient to compare the actions $\hat{a}^*$ and $\hat{a}$:
    \begin{equation*}
        \mathbf{Q}_i^\pi(s, \hat{a}^*) \geq \mathbf{Q}_i^\pi(s, \hat{a})
    \end{equation*}
    \begin{align*}
        &R_i(s, \hat{a}^*) + \gamma \sum_{s'} T(s' \mid s, \hat{a}^*) \textbf{V}_i^\pi(s') - \sum_{v,w} \beta_{wvs} C_{wv}(s, \hat{a}^*) \\
        &\quad \geq R_i(s, \hat{a}) + \gamma \sum_{s'} T(s' \mid s, \hat{a}) \textbf{V}_i^\pi(s') - \sum_{v,w} \beta_{wvs} C_{wv}(s, \hat{a})
    \end{align*}
    Since $\mathbf{V}_i^\pi(s') = V_i^\pi(s')$ and $C_{wv}(s, \hat{a}^*) = 0$, we have:
    \begin{align*}
        &R_i(s, \hat{a}^*) + \gamma \sum_{s'} T(s' \mid s, \hat{a}^*) V_i^\pi(s') - \sum_{v,w} \beta_{wvs} 0 \\
        &\quad \geq R_i(s, \hat{a}) + \gamma \sum_{s'} T(s' \mid s, \hat{a}) V_i^\pi(s') - \sum_{v,w} \beta_{wvs} C_{wv}(s, \hat{a})
    \end{align*}
    \begin{equation*}
        Q_i(s, \hat{a}^*) - Q_i(s, \hat{a}) + \sum_{v,w} \beta_{wvs} C_{wv}(s, \hat{a}) \geq 0
    \end{equation*}
    Let $\beta_s$ be the same constant used for all $\beta_{wvs}$, yielding:
    \begin{equation*}
        Q_i(s, \hat{a}^*) - Q_i(s, \hat{a}) + \beta_{s} \sum_{v,w} C_{wv}(s, \hat{a}) \geq 0
    \end{equation*}
    \begin{equation*}
        \beta_{s} \geq \frac{Q_i(s, \hat{a}) - Q_i(s, \hat{a}^*)}{\sum_{v,w} C_{wv}(s, \hat{a})}
    \end{equation*}
    This inequality must be satisfied for all $\hat{a}$, enforcable with a maximization over actions.
    This results in Equation~\ref{eq:tmdp:lower-bound}.
\end{proof}

\begin{theorem}
    \label{prop:lagrangian-bellman-solves-tmdp-lar-optimization}
    Lagrangian Bellman optimality Equation~\ref{eq:tpo:lagrangian-bellman} solves the optimization in Equation~\ref{eq:tmdp:objective-lar}.
\end{theorem}

\begin{proof}
    By Proposition~\ref{prop:lagrange-multiplier-bound}, there exists a $\beta_{wvs}$ such that the use of the Lagrangian Bellman optimality equation (Equation~\ref{eq:tpo:lagrangian-bellman}) ensures the constraints:
        $C_{wv}(s, a) \leq 0, \forall w, v$,
    are always satsified.

    Since the constraints are satisfied, by construction Equation~\ref{eq:tpo:lagrangian-bellman} has constraint terms $0$, implying $\mathbf{Q}_i^*(s, \hat{a}^*) = Q_i^*(s, \hat{a}^*) - 0 = Q_i^*(s, \hat{a}^*)$.
    Thus it solves the optimization in Equation~\ref{eq:tmdp:bellman-lar-transformed}.

    By Proposition~\ref{prop:lagrangian-bellman-solves-tmdp-lar}, it also solves the optimization problem in Equation~\ref{eq:tmdp:bellman-lar}.
    By construction of Equation~\ref{eq:tmdp:bellman-lar}'s constraints, for all states $s$, $-A_w^*(s, a) \leq \eta_{wv}$, and again $\mathbf{V}_i^*(s, a) = V_i^*(s, a)$.
    Thus, all constraints of the original optimization in Equation~\ref{eq:tmdp:objective-lar} are satified and return the same resulting policy at all states, solving this original optimization.
\end{proof}

\subsection{Multi-Objective Lagrangian Policy Gradient Theorem}

We have changed the objective into a form involving additional ancestor objectives.
Consequently, it is not a given that parameterized policies learning from samples---such as in the case of deep reinforcement learning---is able to converge to the optimal policy.
Here we present a novel policy gradient theorem~\cite{SMSMnips00} for any algorithm using this Lagrangian TMDP objective.
This \emph{multi-objective Lagrangian policy gradient theorem} is presented in Proposition~\ref{prop:multi-objective-lagrangian-policy-gradient-theorem} below.

\begin{proposition}
    \label{prop:multi-objective-lagrangian-policy-gradient-theorem}
    For any TMDP, with an discounted infinite horizon objective (or an average-reward objective), for each objective $i$:
    \begin{equation}
        \label{eq:multi-objective-lagrangian-policy-gradient-theorem}
        \frac{\partial \mathbf{\rho}_i^\pi}{\partial \theta} = \sum_s d^\pi(s) \sum_a \frac{\partial \pi(a \mid s)}{\partial \theta} \mathbf{Q}_i^\pi(s, a).
    \end{equation}
\end{proposition}

\begin{proof}
    We begin by rewriting Equation~\ref{eq:tpo:lagrangian-bellman}:
    \begin{equation*}
        \textbf{V}_i^\pi(s) = \max_{a} \Big( \mathbf{R}_i(s, a)  + \gamma \sum_{s'} T(s' \mid s, a) \textbf{V}_i^\pi(s') \Big)
    \end{equation*}
    with $\mathbf{R}_i(s, a) = R_i(s, a) - \sum_{v,w} \beta_{wvs} C_{wv}(s, a)$.
    Consider the partial derivative of $\mathbf{R}_i$ with respect to policy parameters $\theta$.
    First, $\frac{\partial}{\partial \theta} R_i(s, a) = 0$.
    Second, by definition in Equation~\ref{eq:tmdp:bellman-lar-transformed}, each $C_{wv}$ does not depend on the current policy parameters $\theta$, as it is proportional to the optimal advantages and slacks of ancestors: $C_{wv}(s, a) = \max\{0, {-}A_w^*(s, a) - \eta_{wv}\}$.
    We have $\frac{\partial}{\partial \theta} C_{wv}(s, a) = 0$.
    Thus, $\frac{\partial}{\partial \theta} \mathbf{R}_i(s, a) = 0$.

    The rest of this proof follows directly from the original policy gradient theorem's proof~\cite{SMSMnips00}, using $\mathbf{R}_i$ as the reward and the fact that $\frac{\partial}{\partial \theta} \mathbf{R}_i(s, a) = 0$.
\end{proof}

\begin{algorithm}[t]
    \caption{Topological policy optimization}
    \label{alg:modm:algorithm}
    \begin{algorithmic}[1]
        \Require{$k$: Number of objectives}
        \Require{$E$: DAG of objective relationships}
        \Require{$\delta$: Slacks for each objectives}
        \Require{$\theta^0$: Initial policy parameters}
        \State $\pi_{\theta}, \hat{\mathbf{V}}$ $\gets$ $\pi_{\theta^0}$, $\emptyset$ 
        \For{$i \gets$ \Call{TopologicalSort}{$E$}}
            \State \textit{// $\hat{\mathbf{A}}_i$ is a function of trajectories computed during}
            \State \textit{// PPO's rollouts; it uses the learned ancestral}
            \State \textit{// critics' set $\hat{\mathbf{V}}$ as a parameter}
            \State $\hat{\mathbf{A}}_i$ $\gets$ Eq.~\ref{eq:tpo:generalized-lagrangian-advantage-estimation} using this iteration's $\hat{\mathbf{V}}$ in $C_{wv}$
            \State $\pi_{\theta}, \hat{V}_i$ $\gets$ \Call{PPO}{$\pi_{\theta}$, $\hat{\mathbf{A}}_i$} \Comment{Use $\hat{\mathbf{A}}_i$ for PPO's adv.}
            \State $\hat{\mathbf{V}}$ $\gets$ $\hat{\mathbf{V}} \cup \{\hat{V}_i\}$ \Comment{Extend critic set for descendants}
        \EndFor
        \State \Return $\pi_{\theta}$, $\hat{\mathbf{V}}$ 
    \end{algorithmic}
\end{algorithm}


\begin{figure*}[t]
    \centering
    \setlength{\tabcolsep}{0pt}         
    \renewcommand{\arraystretch}{1.1}   
    \begin{tabular}{cccc}
        \scalebox{0.9}[0.9]{%
            \begin{tikzpicture}
                \node[latent] (1) {$M$} ; %
                \node[latent, right=of 1, xshift=-1em] (2) {$A$} ; %
                \node[latent, right=of 2, xshift=-1em] (3) {$G$} ; %
                \horizontaledge{1}{$\eta_1$}{2} ; %
                \horizontaledge{2}{$\eta_2$}{3} ; %
                \node[above=of 2, yshift=0em] (label) {\large (a)} ; %
                \node[below=of 2, yshift=-0em] (blank) {} ; %
            \end{tikzpicture}
        } &
        \begin{tikzpicture}
            \begin{axis}[
                    xmin=0, xmax=100,
                    ymin=-800, ymax=0,
                    xtick distance = 50,
                    ytick distance = 200,
                    grid = both,
                    major grid style = {lightgray},
                    minor grid style = {lightgray!25},
                    width = 5cm,
                    height = 4cm,
                    xlabel = {$\eta_1 = \eta_2$}, 
                    ylabel = {$\mathbf{V}$},
                ]
                \draw[ultra thick,red,-] (0, -133.65) circle (1pt) -- (100, -311.16) circle (1pt); %
                \draw[ultra thick,blue,-] (0, -114.47) circle (1pt) -- (100, -310.06) circle (1pt); %
                \draw[ultra thick,green,-] (0, -773.65) circle (1pt) -- (100, -596.74) circle (1pt); %
            \end{axis}
        \end{tikzpicture} &
        \scalebox{0.9}[0.9]{%
            \begin{tikzpicture}
                \node[latent] (1) {$A$} ; %
                \node[latent, right=of 1, xshift=-1em] (2) {$M$} ; %
                \node[latent, right=of 2, xshift=-1em] (3) {$G$} ; %
                \horizontaledge{1}{$\eta_1$}{2} ; %
                \horizontaledge{2}{$\eta_2$}{3} ; %
                \node[above=of 2, yshift=0em] (label) {\large (b)} ; %
                \node[below=of 2, yshift=-0em] (blank) {} ; %
            \end{tikzpicture}
        } &
        \begin{tikzpicture}
            \begin{axis}[
                    xmin=0, xmax=100,
                    ymin=-800, ymax=0,
                    xtick distance = 50,
                    ytick distance = 200,
                    grid = both,
                    major grid style = {lightgray},
                    minor grid style = {lightgray!25},
                    width = 5cm,
                    height = 4cm,
                    xlabel = {$\eta_2$}, 
                    ylabel = {$\mathbf{V}$},
                ]
                \draw[ultra thick,red,-] (0, -72.94) circle (1pt) -- (100, -186.35) circle (1pt); %
                \draw[ultra thick,blue,-] (0, -65.27) circle (1pt) -- (100, -111.21) circle (1pt); %
                \draw[ultra thick,green,-] (0, -738.6) circle (1pt) -- (100, -600.74) circle (1pt); %
            \end{axis}
        \end{tikzpicture} \\
        \scalebox{0.9}[0.9]{%
            \begin{tikzpicture}
                \node[latent] (1) {$G$} ; %
                \node[latent, right=of 1, xshift=-1em] (2) {$M$} ; %
                \node[latent, right=of 2, xshift=-1em] (3) {$A$} ; %
                \horizontaledge{1}{$\eta_1$}{2} ; %
                \horizontaledge{2}{$\eta_2$}{3} ; %
                \node[above=of 2, yshift=0em] (label) {\large (c)} ; %
                \node[below=of 2, yshift=-0em] (blank) {} ; %
            \end{tikzpicture}
        } &
        \begin{tikzpicture}
            \begin{axis}[
                    xmin=0, xmax=100,
                    ymin=-1000, ymax=-200,
                    xtick distance = 50,
                    ytick distance = 200,
                    grid = both,
                    major grid style = {lightgray},
                    minor grid style = {lightgray!25},
                    width = 5cm,
                    height = 4cm,
                    xlabel = {$\eta_1$}, 
                    ylabel = {$\mathbf{V}$},
                ]
                \draw[ultra thick,red,-] (0, -364.38) circle (1pt) -- (100, -266.71) circle (1pt); %
                \draw[ultra thick,blue,-] (0, -288.76) circle (1pt) -- (100, -237.21) circle (1pt); %
                \draw[ultra thick,green,-] (0, -957.6) circle (1pt) -- (100, -957.46) circle (1pt); %
            \end{axis}
        \end{tikzpicture} &
        \scalebox{0.9}[0.9]{%
            \begin{tikzpicture}
                \node[latent] (3) {$G$} ; %
                \node[latent, above left=of 3, xshift=0.5em, yshift=-1.5em] (1) {$A$} ; %
                \node[latent, above right=of 3, xshift=-0.5em, yshift=-1.5em] (2) {$M$} ; %
                \backwardslantedgebelow{1}{$\eta_{1}$}{3} ; %
                \forwardslantedgebelow{2}{$\eta_{2}$}{3} ; %
                \node[above=of 3, yshift=0em] (label) {\large (d)} ; %
                \node[below=of 2, yshift=-2em] (blank) {} ; %
            \end{tikzpicture}
        } &
        \begin{tikzpicture}
            \begin{axis}[
                    xmin=0, xmax=100,
                    ymin=-800, ymax=0,
                    xtick distance = 50,
                    ytick distance = 200,
                    grid = both,
                    major grid style = {lightgray},
                    minor grid style = {lightgray!25},
                    width = 5cm,
                    height = 4cm,
                    xlabel = {$\eta_1 = \eta_2$},
                    ylabel = {$\mathbf{V}$},
                ]
                \draw[ultra thick,red,-] (0, -251.67) circle (1pt) -- (100, -519.57) circle (1pt); %
                \draw[ultra thick,blue,-] (0, -185.26) circle (1pt) -- (100, -509.55) circle (1pt); %
                \draw[ultra thick,green,-] (0, -769.72) circle (1pt) -- (100, -670.23) circle (1pt); %
            \end{axis}
        \end{tikzpicture}
    \end{tabular}
    \caption{Results for four different DAGs and slack assignments on the robot navigation domain. The x-axis denotes slack $\eta$. The y-axis denotes values for all three objectives: $V_\text{avoid}$ (red), $V_\text{monitor}$ (blue), and $V_\text{goal}$ (green). Note that the units of the values for the three objectives are distinct from one another.}
    \label{fig:experiments:dag-and-slack}
    \vspace{-10pt}
\end{figure*}
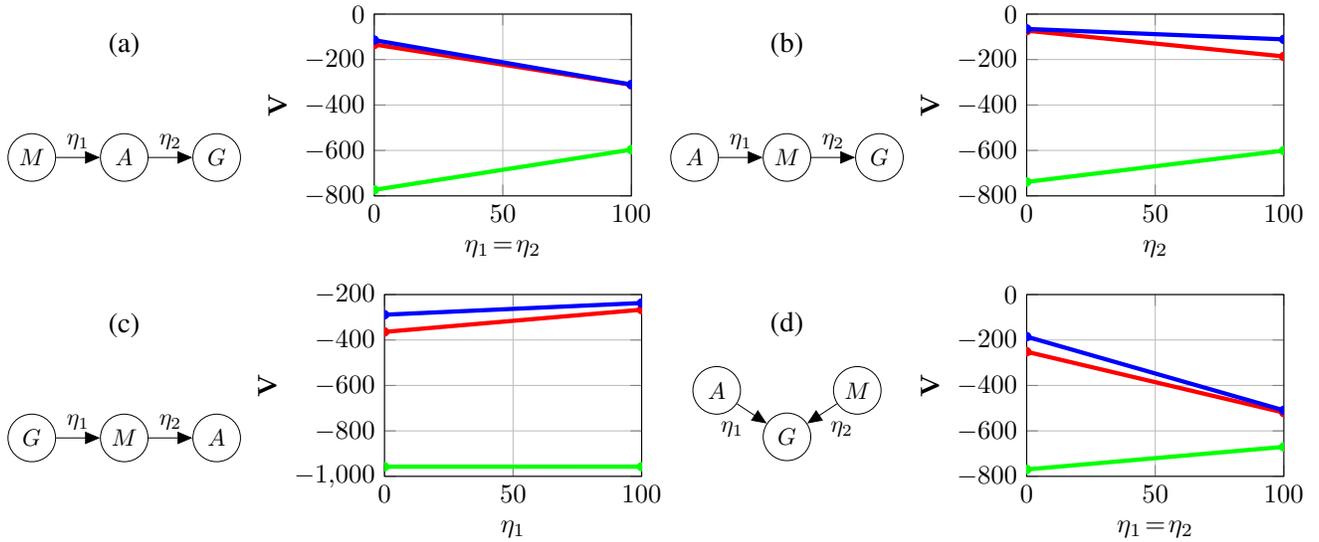

\section{Topological Policy Optimization}
\label{sec:topological-policy-optimization}

Proximal policy optimization (PPO)~\cite{SWDRKarxiv17} uses a (clipped) surrogate objective, computing an approximate policy gradient over batch trajectories $\tau$ of length $h$, and using a generalized advantage estimation.
We apply our Lagrangian Bellman equation results from the previous section to PPO.
The following equations refer to this novel contribution.
At time step $t$, the objective $i$'s \textbf{Lagrangian surrogate objective} is:
\begin{equation}
    \label{eq:tpo:lagrangian-surrogate-objective}
    \mathbf{L}_i(\theta) = \hat{\mathbb{E}}^t \Big[ \log \pi_{\theta}(a^t | s^t) \hat{\mathbf{A}}_i^t \Big] 
\end{equation}
with $\hat{\mathbf{A}}_i^t$ denoting advantages computed by \textbf{generalized Lagrangian advantage estimation} (Equation~\ref{eq:tpo:generalized-lagrangian-advantage-estimation} below).

The generalized Lagrangian advantage estimation defines $\hat{\mathbf{A}}_i^t = \mathbf{\delta}_i^t +(\gamma \lambda) \mathbf{\delta}_i^{t+1} + \cdots + (\gamma \lambda)^{h-t+1} \mathbf{\delta}_i^{h-1}$ with single advantage estimates $\mathbf{\delta}_i^k = \mathbf{r}_i^k + \gamma \mathbf{V}_i^\pi(s^{k+1}) - \mathbf{V}_i^\pi(s^k)$.
The sample-based estimate for this expression is computed by:
\begin{align}
    \hat{\mathbf{A}}_i^t &= \sum_{k=t}^{h-1} (\gamma \lambda)^k \mathbf{\delta}_i^k = \sum_{k=t}^{h-1} (\gamma \lambda)^k \Big( \mathbf{r}_i^k + \gamma \mathbf{V}_i^\pi(s^{k+1}) - \mathbf{V}_i^\pi(s^k) \Big) \nonumber \\
                         &= \sum_{k=t}^{h-1} (\gamma \lambda)^k \Big( r_i^k - \sum_{v,w} \beta_{wvs^k} C_{wv}(s^k, a^k) \nonumber \\
                         &\quad \quad \quad \quad + \gamma (V_i^\pi(s^{k+1}) - \sum_{v,w} \beta_{wvs^{k+1}} C_{wv}(s^{k+1}, a^{k+1})) \nonumber \\
                         &\quad \quad \quad \quad - (V_i^\pi(s^k) - \sum_{v,w} \beta_{wvs^k} C_{wv}(s^k, a^k)) \Big) \nonumber \\
                         &= \sum_{k=t}^{h-1} (\gamma \lambda)^k \Big( r_i^k + \gamma V_i^\pi(s^{k+1}) - V_i^\pi(s^k) \nonumber \\
                         &\quad \quad \quad \quad - \gamma \sum_{v,w} \beta_{wvs^{k+1}} C_{wv}(s^{k+1}, a^{k+1}) \Big) \label{eq:tpo:generalized-lagrangian-advantage-estimation}
\end{align}
with $C_{wv}$ incorporating the ancestral advantages $A_{wv}$ as defined by Equation~\ref{eq:tmdp:bellman-lar-transformed}.

Additional components of PPO and similar algorithms may be used, which are applied inside of the \Call{PPO}{} routine in line 6, Algorithm \ref{alg:modm:algorithm}.
Specifically, we employ PPO's policy gradient entropy term and its clipped ratio for the policy.





%
%
%
%
%
%
%
%
%

\section{Experiments}
\label{sec:experiments}

In this section, we consider experiments in the \emph{multi-objective robot navigation}~\cite{WCicra22} domain, which is fully implemented on a real robot acting in an actual household environment.
Importantly, the main contribution of this paper remains the multi-objective policy gradient's theoretical results and formulation, rather than the implementation of the approach in TPO and a real world navigation domain.
These results are included to provide evidence of the approach's usefulness and build intuitions in modeling objective structures and behavioral customizability using slack.

\subsection{Multi-Objective Robot Navigation Domain}

We consider navigation domains in which a robot is provided with the tuple $\langle \mathcal{M}, \mathcal{L}^0, \mathcal{L}^g, \mathcal{R}^a, \mathcal{R}^m \rangle$.
$\mathcal{M} \in \{0, 1\}^{m \times n}$ is a map, described here as a occupancy grid (or binary matrix), with a starting location $\mathcal{L}^0 \in \mathcal{M}$ and a goal location $\mathcal{L}^g \in \mathcal{M}$.
All this information and the state transitions, which describe the robot's movement dynamics and its interaction with walls or obstacles in the map, are unknown a priori.

The \textit{goal} objective is to navigate from $\mathcal{L}^0$ to $\mathcal{L}^g$ as fast as possible.
The robot receives a $-1$ for any non-goal state, including when interacting with walls or obstacles.
Once the robot reaches the goal $\mathcal{L}^g$, its navigation ends.

Two additional objectives are also provided.
The \textit{avoid} objective is to penalize entering a rectangular region $\mathcal{R}^a$.
The robot receives a $-1$ for each time step in this region.
The \textit{monitor} objective is to encourage entering into a rectangular region $\mathcal{R}^m$.
The robot receives a $+1$ for each time step in this region.
For both objectives, any interaction with walls or obstacles is still met with a penalty of $-1$.

This domain builds on prior work on multi-objective (PO)MDPs for home healthcare robots~\cite{WCicra22}.
In particular, the navigation of a home healthcare robot can be tailored to different homes via the map $\mathcal{M}$; delivering medicine through the assignment of $\mathcal{L}^0$ and $\mathcal{L}^g$; the preference of the human(s) or patients to avoid traversing rooms via $\mathcal{R}^a$; the preference to monitor rooms while navigating via $\mathcal{R}^m$; and so on.

\begin{figure*}[t]
    \centering
    \includegraphics[height=3.5cm,width=3.4cm]{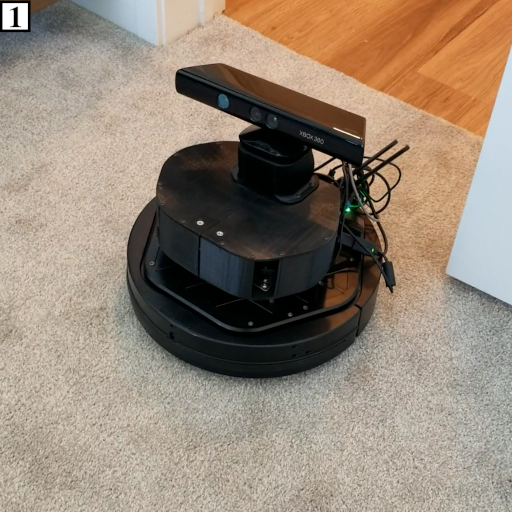}
    \includegraphics[height=3.5cm,width=3.4cm]{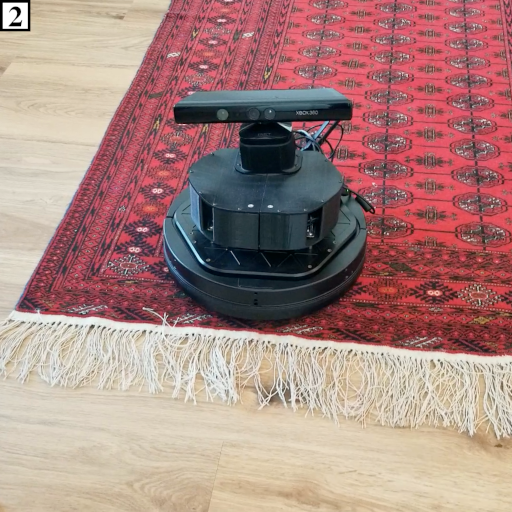}
    \includegraphics[height=3.5cm,width=3.4cm]{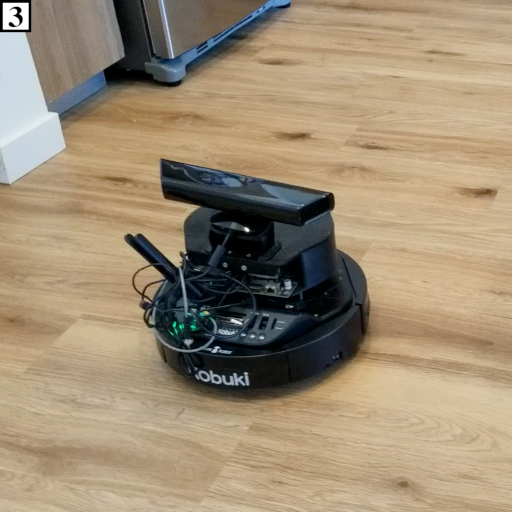}
    \includegraphics[height=3.5cm,width=3.4cm]{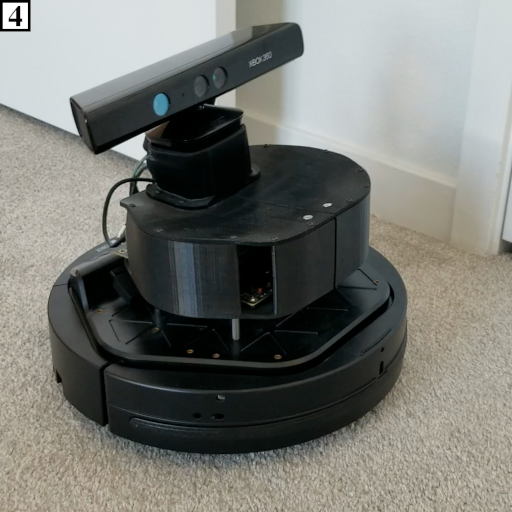}
    \includegraphics[height=3.5cm,width=3.4cm]{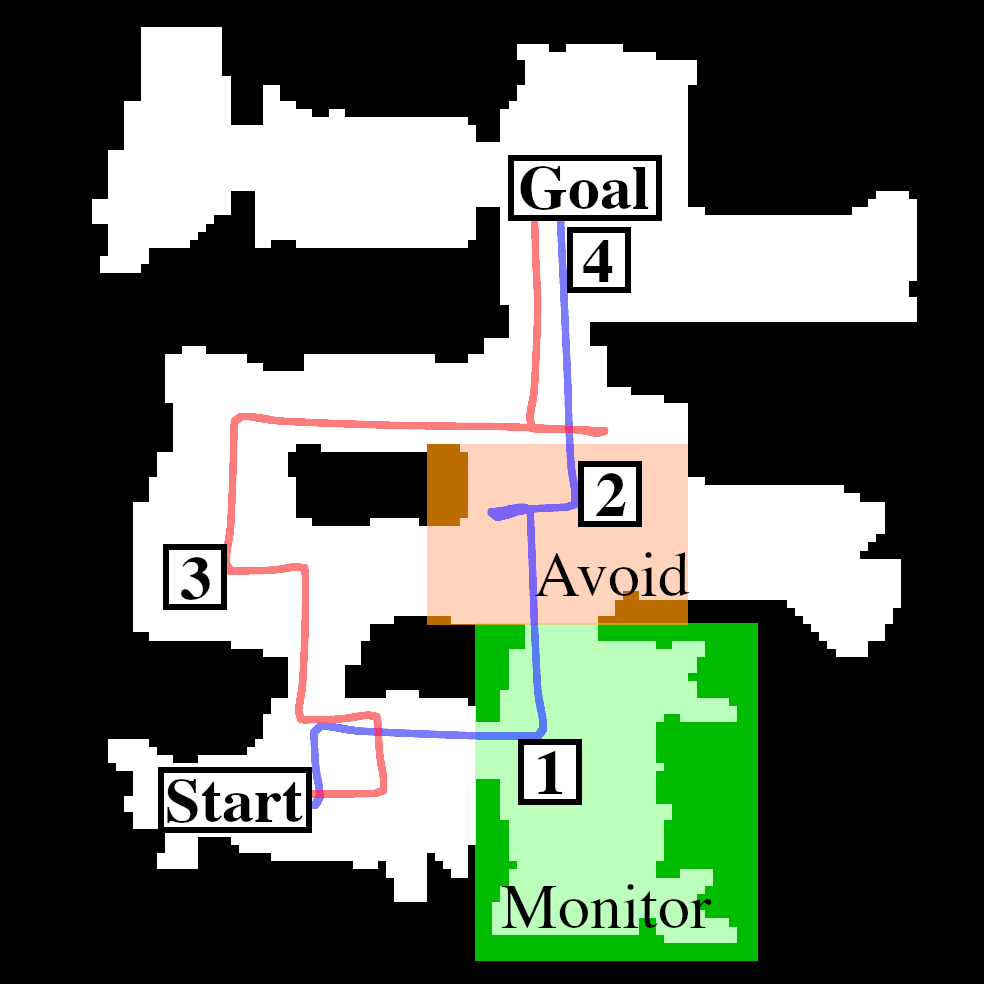}
    \caption{Experiments that implement \emph{home navigation} on a real robot in an actual household environment. This uses the $M \rightarrow A \rightarrow G$ DAG with two paths: (1) constrained-to-monitor (blue path) with $\eta_1 = \eta_2 = 0$, and (2) constrained-to-avoid (red path) with $\eta_1 = 100$ and $\eta_2 = 0$. The start point (S), \textit{goal} point (G), \textit{avoid} region, and \textit{monitor} region are shown.}
    \label{fig:experiments:robot}
\end{figure*}

\subsection{Experimental Setting}

The thesis of this paper is to propose a novel multi-objective policy gradient formulation for the new TMDP model, for which no algorithm currently exists for continuous state spaces.
Therefore, the theoretical result and formulation must ideally be evaluated by comparing its efficacy at modeling different permutations of multiple objective DAGs and slack assignments.
Thus, the primary metrics must be the values of all three objectives: $V_\text{avoid}$, $V_\text{monitor}$, and $V_\text{goal}$.
The results indicate how effective the approach is at capturing the slack and the preferences encoded by each of the DAGs.

The implementation of this theoretical approach is the approximate TPO algorithm.
For each of the configurations of DAG and slacks, the algorithm is trained for $300000$ iterations.
The experiments compute the value for each objective via Monte Carlo simulations using the agent's final policy network.
Analysing convergence is left to future work.
The approach is implemented in Julia 1.6.1 on Ubuntu 18.04.

\subsection{Results and Discussion}

Figure~\ref{fig:experiments:dag-and-slack} provides results of four distinct constraint DAGs, with each varying the slack values.
This figure illustrates the effect of the topological structure of the constraint DAG and the slack assignments on the three objectives' values.

In Figures~\ref{fig:experiments:dag-and-slack} (a) and (b), we see the desired effect: increasing the slack of the constraining objectives \textit{avoid} ($A$) and \textit{monitor} ($M$) reduces their value and enables the \textit{goal} ($G$) objective to improve its value.
Figure~\ref{fig:experiments:dag-and-slack} (d) shows a similar effect, except with a different fan DAG structure.
In this case, the result DAG structure enables a more drastic change.
In Figure~\ref{fig:experiments:dag-and-slack} (c), the \textit{goal} is the first objective.
We observe now that by increasing the slack of the $G$, the other two objectives ($A$ and $M$) are able to increase their values.

Figure~\ref{fig:experiments:robot} demonstrates the implementation of the approach on an actual robot acting in a real household environment.
With zero slack, the primary \textit{monitor} objective is favored, forcing the \textit{avoid} objective to experience a penalty.
However when provided slack, the \textit{avoid} objective is able to direct the path away from the region while navigating to the goal.

\section{Conclusion}

This paper presents a policy gradient approach for capturing rich multi-objective preference structures in reinforcement learning.
While primarily a theoretical paper, the method is demonstrated in a real robot navigation domain.
With this paper's established theoretical foundation, future work will introduce applications of this approach for a range of reinforcement learning problems including formulations of curriculum learning and imitation learning.

\newpage

\bibliographystyle{style/IEEEtran}
\bibliography{references}

\end{document}